\documentclass[12pt]{article}
\usepackage{graphicx}
\graphicspath{{./figure/}}
\usepackage{amsmath}
\usepackage{amsthm}
\usepackage{amssymb}
\usepackage[square,numbers]{natbib}
\usepackage{booktabs}
\usepackage{url}
\usepackage[top=30truemm,bottom=30truemm,left=25truemm,right=25truemm]{geometry}
\usepackage{bm}
\usepackage{subfig}
\usepackage[misc]{ifsym}

\newtheorem{definition}{Definition}
\newtheorem{proposition}{Proposition}

\begin{document}
\title{Linearly Constrained Weights: Reducing Activation Shift for
    Faster Training of Neural Networks}
\author{Takuro Kutsuna$\/^*$}
\date{\normalsize Toyota Central R\&D Labs., Inc. \\ $\/^*$kutsuna@mosk.tytlabs.co.jp}
\maketitle
\begin{abstract}
    In this paper, we first identify \textit{activation shift}, a simple but remarkable
    phenomenon in a
    neural network in which
    the preactivation value of a neuron has non-zero mean that depends on the angle
    between the weight vector of the neuron and the mean of the activation
    vector in the previous layer.
    We then propose \textit{linearly constrained weights (LCW)} to reduce
    the activation shift in both fully connected and convolutional layers.
    The impact of reducing the activation shift in a neural network is studied
    from the perspective of how the variance of variables in the network
    changes through layer operations in both forward and backward chains.
    We also discuss its relationship to the vanishing gradient problem.
    Experimental results show that LCW enables a deep feedforward network
    with sigmoid activation functions to be trained efficiently
    by resolving the vanishing gradient problem.
    Moreover, combined with batch normalization, LCW improves
    generalization performance of both feedforward and
    convolutional networks.
\end{abstract}
\section{Introduction}
\renewcommand{\thefootnote}{\fnsymbol{footnote}}
\footnote[0]{\scriptsize{The final authenticated publication is available online at \url{https://doi.org/10.1007/978-3-030-46147-8_16}}}
\renewcommand{\thefootnote}{\arabic{footnote}}
Neural networks with a single hidden layer have been shown to be universal
approximators~\cite{hornik1989multilayer,irie1988capabilities}.
However, an exponential number of neurons may be necessary to
approximate complex functions.
One solution to this problem is to use more hidden layers.
The representation power of a network increases exponentially with the
addition of layers~\cite{pmlr-v49-eldan16,pmlr-v49-telgarsky16}.
Various techniques have been proposed for training deep nets, that is,
neural networks with many hidden layers, such as
layer-wise pretraining~\cite{Hinton504}, rectified linear
units~\cite{5459469,icml2010_NairH10}, residual
structures~\cite{he2016deep}, and normalization
layers~\cite{JMLR:v17:gulchere16a,icml2015_ioffe15}.

In this paper, we first identify the \textit{activation shift}
that arises in the calculation of the preactivation value of a neuron.
The preactivation value is calculated as the dot product of
the weight vector of a neuron and an activation vector in the previous layer.
In a neural network, an activation vector in a layer can be viewed as a
random vector whose distribution is determined by the input distribution and the
weights in the preceding layers.
The preactivation of a neuron then has a \textit{non-zero mean} depending on
the angle between the weight vector of the neuron
and the mean of the activation vector in the previous layer.
The angles are generally different according to the neuron, indicating
that neurons have distinct mean values, even those in the same layer.

We propose the use of so-called \textit{linearly constrained
    weights~(LCW)} to resolve the activation shift in both fully connected and
convolutional layers.
An LCW is a weight vector subject to the constraint that the sum of
its elements is zero.
We investigate the impact of resolving activation shift in a neural network
from the perspective of how the variance of variables in a neural
network changes according to
layer operations in both forward and backward directions.
Interestingly, in a fully connected layer in which the activation shift has been
resolved by LCW, the variance is amplified by the same rate in both
forward and backward chains.
In contrast, the variance is more amplified in the forward chain
than in the backward chain when activation shift occurs in the layer.
This asymmetric characteristic is suggested to be a cause of the vanishing gradient in
feedforward networks with sigmoid activation functions.
We experimentally demonstrate that we can successfully train a deep
feedforward network with sigmoid activation functions by reducing the activation shift
using LCW.
Moreover, our experiments suggest that LCW improves generalization
performance of both feedforward and convolutional networks when combined
with batch normalization (BN)~\cite{icml2015_ioffe15}.

In Section~\ref{sec:act-shift}, we give a general definition of activation
shift in a neural network.
In Section~\ref{sec:line-constr-weights}, we propose LCW as an approach to
reduce activation shift and present a technique to efficiently train a network
with LCW.
In Section~\ref{sec:netw-arch-weight} we study the impact of removing
activation shift in a neural network from the perspective of variance
analysis and then discuss its relationship to the vanishing gradient problem.
In Section~\ref{sec:related-work}, we review related work.
We present empirical results in Section~\ref{sec:experiments} and
conclude the study in Section~\ref{sec:conclusion}.

\section{Activation Shift}
\label{sec:act-shift}
We consider a standard multilayer perceptron (MLP). For simplicity,
the number of neurons~$m$ is assumed to be the same in all layers.
The activation vector in layer~$l$ is denoted
by~$\bm{a}^l = \left(a_1^l, \ldots, a_m^l\right)^\top \in \mathbb{R}^m$.
The input vector to the network is denoted by~$\bm{a}^0$.
The weight vector of the $i$-th neuron in layer~$l$ is denoted
by~$\bm{w}_i^{l} \in \mathbb{R}^m$. It is generally assumed
that~$\|\bm{w}_i^l\| > 0$.
The activation of the $i$-th neuron in layer~$l$ is given
by~$a_i^{l} = f\left(z_i^{l}\right)$
and~$z_i^{l} = \bm{w}_i^{l} \cdot \bm{a}^{l-1} + b_i^{l}$,
where~$f$ is a nonlinear activation function, $b_i^{l} \in \mathbb{R}$ is the
bias term, and~$z_i^{l} \in \mathbb{R}$ denotes
the preactivation value.
Variables~$z_i^{l}$ and~$a_i^{l}$ are regarded as random variables whose
distributions are determined by the distribution of the input
vector~$\bm{a}^0$, given the weight vectors and the bias terms in the
preceding layers.

We introduce activation shift using the simple example
shown in Fig.~\ref{fig:act_shift_example}.
Fig.~\ref{fig:act_shift_example}\subref{fig:plot_W} is a heat map
representation of a weight
matrix~$\bm{W}^l \in \mathbb{R}^{100\times 100}$, whose $i$-th row
vector represents~$\bm{w}_i^l$.
In Fig.~\ref{fig:act_shift_example}\subref{fig:plot_W}, each element
of~~$\bm{W}^l$ is independently drawn from a uniform random
distribution in the range~$(-1, 1)$.
Fig.~\ref{fig:act_shift_example}\subref{fig:plot_A} shows an
activation matrix~$\bm{A}^{l-1} \in \mathbb{R}^{100\times 100}$, whose $j$-th
column vector represents the activation vector corresponding to the
$j$-th sample in a minibatch.
Each element of~$\bm{A}^{l-1}$ is randomly sampled from the range~$(0, 1)$.
We multiply~$\bm{W}^l$ and~$\bm{A}^{l-1}$ to obtain the preactivation
matrix~$\bm{Z}^l$, whose $i$-th row vector represents preactivation
values of the $i$-th neuron in layer~$l$, which is shown in
Fig.~\ref{fig:act_shift_example}\subref{fig:plot_WA}.
It is assumed that bias terms are all zero.
\begin{figure}[t]
    \centering
    \subfloat[Weight~$\bm{W}^l$.]{
        \includegraphics[width=0.3\linewidth, clip]{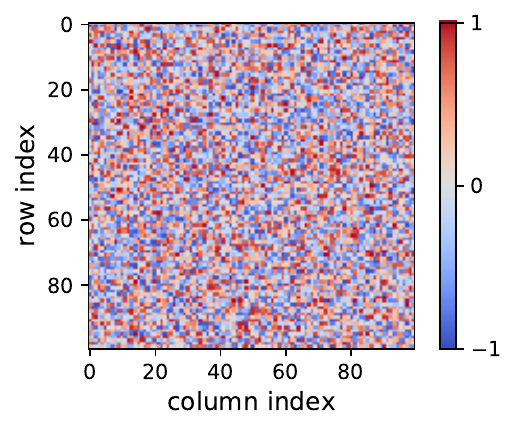}
        \label{fig:plot_W}}
    \hfill%
    \subfloat[Activation~$\bm{A}^{l-1}$.]{
        \includegraphics[width=0.29\linewidth, clip]{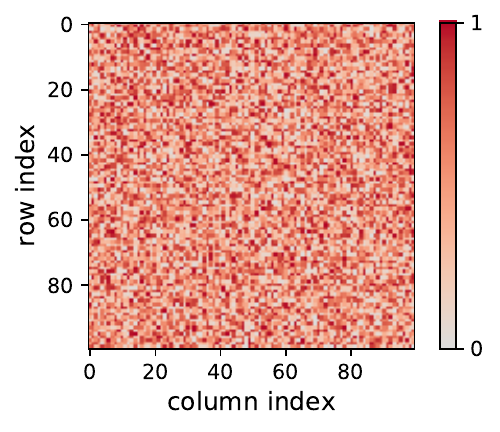}
        \label{fig:plot_A}}
    \hfill%
    \subfloat[Preactivation~$\bm{Z}^l$.]{
        \includegraphics[width=0.3\linewidth,  clip]{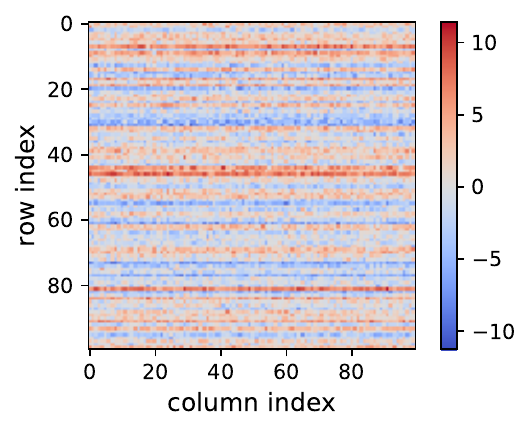}
        \label{fig:plot_WA}}
    \caption{Activation shift causes a horizontal stripe pattern in
        preactivation~$\bm{Z}^l = \bm{W}^l \bm{A}^{l-1}$,
        in which each element of~$\bm{W}^l$ and~$\bm{A}^{l-1}$ is randomly generated
        from the range~$(-1,1)$ and~$(0,1)$, respectively.}
    \label{fig:act_shift_example}
\end{figure}
Unexpectedly, a horizontal stripe pattern appears in the
heat map of~$\bm{Z}^l$ even though both~$\bm{W}^l$ and~$\bm{A}^{l-1}$
are randomly generated.
This pattern is attributed to the activation shift, which is
defined as follows:
\begin{definition}
    $\mathcal{P}_{\gamma}$ is an~$m$-dimensional probability distribution
    whose expected value is~$\gamma \bm{1}_m$, where~$\gamma \in \mathbb{R}$ and $\bm{1}_m$ is an
    $m$-dimensional vector whose elements are all one.
\end{definition}
\begin{proposition} \label{prop1}
    Assume that the activation vector~$\bm{a}^{l-1}$
    follows~$\mathcal{P}_{\gamma}$.
    Given a weight vector~$\bm{w}_i^l \in \mathbb{R}^m$
    such that~$\|\bm{w}_i^l\| > 0$,
    the expected value of~$\bm{w}_i^l \cdot \bm{a}^{l-1}$
    is~$|\gamma| \sqrt{m} \|\bm{w}_i^l\| \cos \theta_i^l$, where~$\theta_i^l$ is the angle
    between~$\bm{w}_i^l$ and~$\bm{1}_m$.
\end{proposition}
Proofs of all propositions are provided in Appendix~A in the
supplementary material.
\begin{definition}
    From Proposition~\ref{prop1}, the expected value of~$\bm{w}_i^l \cdot \bm{a}^{l-1}$ depends
    on~$\theta_i^l$ as long as~$\gamma \neq 0$. The
    distribution of~$\bm{w}_i^l \cdot \bm{a}^{l-1}$ is then biased
    depending on~$\theta_i^l$; this is called \textit{activation shift}.
\end{definition}
In Fig.~\ref{fig:act_shift_example}, each column vector
of~$\bm{A}^{l-1}$ follows~$\mathcal{P}_{\gamma}$ with $\gamma=0.5$.
Therefore, the $i$-th row of~$\bm{Z}^l$ is biased according to the angle between~$\bm{w}_i^l$
and~$\bm{1}_m$.
We can generalize Proposition~\ref{prop1} for any $m$-dimensional
distribution~$\hat{\mathcal{P}}$ instead of~$\mathcal{P}_{\gamma}$
by stating that the distribution of~$\bm{w}^l \cdot \hat{\bm{a}}^{l-1}$
is biased according to~$\hat{\theta}_i^l$ unless~$\|\hat{\bm{\mu}}\| = 0$
as follows:
\begin{proposition} \label{prop2}
    Assume that the activation vector~$\hat{\bm{a}}^{l-1}$follows an $m$-dimensional probability
    distribution~$\hat{\mathcal{P}}$ whose expected value is~$\hat{\bm{\mu}} \in \mathbb{R}^m$.
    Given~$\bm{w}_i^l \in \mathbb{R}^m$ such that~$\|\bm{w}_i^l\| > 0$, it follows
    that~$E(\bm{w}_i^l \cdot \hat{\bm{a}}^{l-1}) = \|\bm{w}_i^l\|\ \|\hat{\bm{\mu}}\| \cos \hat{\theta}_i^l$
    if~$\|\hat{\bm{\mu}}\| > 0$; otherwise, $E(\bm{w}_i^l \cdot \hat{\bm{a}}^{l-1}) = 0$,
    where $\hat{\theta}_i^l$ is the angle between~$\bm{w}_i^l$ and~$\hat{\bm{\mu}}$.
\end{proposition}
From Proposition~\ref{prop2},
if~$\bm{a}^{l-1}$
follows~$\hat{\mathcal{P}}$ with the mean vector~$\hat{\bm{\mu}}$ such
that~$\|\hat{\bm{\mu}}\| > 0$,
the preactivation~$z_i^{l}$ is biased
according to the angle between~$\bm{w}_i^{l}$
and~$\hat{\bm{\mu}}$.

Note that differences in~$E(z_i^l)$ are not resolved by simply
introducing bias terms~$b_i^l$, because~$b_i^l$ are optimized to
decrease the training loss function and not to absorb the differences
between~$E(z_i^l)$ during the network training. Our experiments suggest that
pure MLPs with several hidden layers are not trainable even though they incorporate bias terms.
We also tried to initialize~$b_i^l$ to absorb the difference
in~$E(z_i^l)$ at the beginning of the training, though it was
unable to train the network, especially when the network has many hidden layers.

\section{Linearly Constrained Weights}
\label{sec:line-constr-weights}
There are two approaches to reducing activation shift in a neural network.
The first one is to somehow make the expected value of the activation
of each neuron close to zero, because activation shift does not occur
if~$\|\hat{\bm{\mu}}\| = 0$ from Proposition~\ref{prop2}.
The second one is to somehow regularize the angle
between~$\bm{w}_i^{l}$ and~$E\left(\bm{a}^{l-1}\right)$.
In this section, we propose a method to reduce activation shift in a neural
network using the latter approach.
We introduce~$\mathcal{W}_{\text{LC}}$ as follows:
\begin{definition}
    $\mathcal{W}_{\text{LC}}$ is a subspace in $\mathbb{R}^m$ defined
    by
    \begin{align*}
        \mathcal{W}_{\text{LC}} := \left\{\bm{w} \in \mathbb{R}^m \ | \  \bm{w}
        \cdot \bm{1}_m= 0 \right\}.
    \end{align*}
    We call weight vector~$\bm{w}_i^{l}$ in~$\mathcal{W}_{\text{LC}}$
    the \textit{linearly constrained weights (LCWs)}.
\end{definition}
The following holds for~$\bm{w} \in \mathcal{W}_{\text{LC}}$:
\begin{proposition} \label{prop3}
    Assume that the activation vector~$\bm{a}^{l-1}$ follows~$\mathcal{P}_{\gamma}$.
    Given~$\bm{w}_i^l \in \mathcal{W}_{\text{LC}}$ such that~$\|\bm{w}_i^l\| > 0$,
    the expected value of~$\bm{w}_i^l \cdot \bm{a}^{l-1}$ is zero.
\end{proposition}
Generally, activation vectors in a network do
not follow~$\mathcal{P}_\gamma$, and consequently, LCW cannot resolve
the activation shift perfectly.
However, we experimentally observed that
the activation vector approximately follows~$\mathcal{P}_\gamma$ in each
layer.
Fig.~\ref{fig:boxplot_activation_lcw_159}\subref{fig:boxplot_activation_lcw_159_init}
shows boxplot summaries of~$a_i^l$ in
a 10-layer sigmoid MLP with LCW, in which the weights of the
network were initialized using the method that will be
explained in Section~\ref{sec:netw-arch-weight}.
We used a minibatch of samples in the CIFAR-10
dataset~\cite{krizhevsky2009learning} to evaluate the distribution of~$a_i^l$.
In the figure, the 1\%, 25\%,
50\%, 75\%, and 99\% quantiles are displayed as whiskers or boxes.
We see that~$a_i^l$ distributes around~$0.5$ in each neuron,
which suggests that~$\bm{a}^{l} \sim \mathcal{P}_\gamma$
approximately holds in every layer.
We also observed the distribution of~$a_i^l$ after 10 epochs of
training, which are shown in
Fig.~\ref{fig:boxplot_activation_lcw_159}\subref{fig:boxplot_activation_lcw_159_after10epochs}.
We see that~$\bm{a}^{l}$ are less likely follow $\mathcal{P}_\gamma$,
but~$a_i^l$ takes various values in each neuron.
In contrast, if we do not apply LCW to the network, the variance of~$a_i^l$
rapidly shrinks through layers immediately after the initialization as
shown in Fig.~\ref{fig:boxplot_activation_159}, in which weights
are initialized by the method in~\cite{glorot2010understanding}.
Experimental results in Section~\ref{sec:experiments} suggest that we can
train MLPs with several dozens of layers very efficiently by applying
the LCW. The effect of resolving the activation shift by applying LCW
will be theoretically analyzed in Section~\ref{sec:netw-arch-weight}.
\begin{figure}[t]
    \centering
    \subfloat[Immediately after the initialization.]{
        \includegraphics[width=\linewidth, clip]{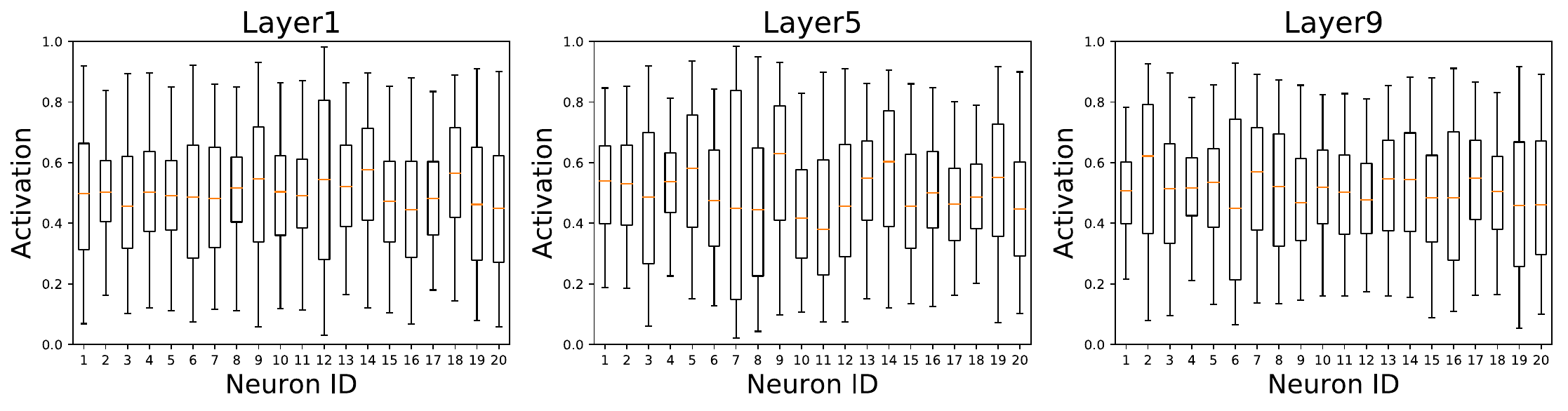}
        \label{fig:boxplot_activation_lcw_159_init}}
    \hfill%
    \subfloat[After 10 epochs training.]{
        \includegraphics[width=\linewidth, clip]{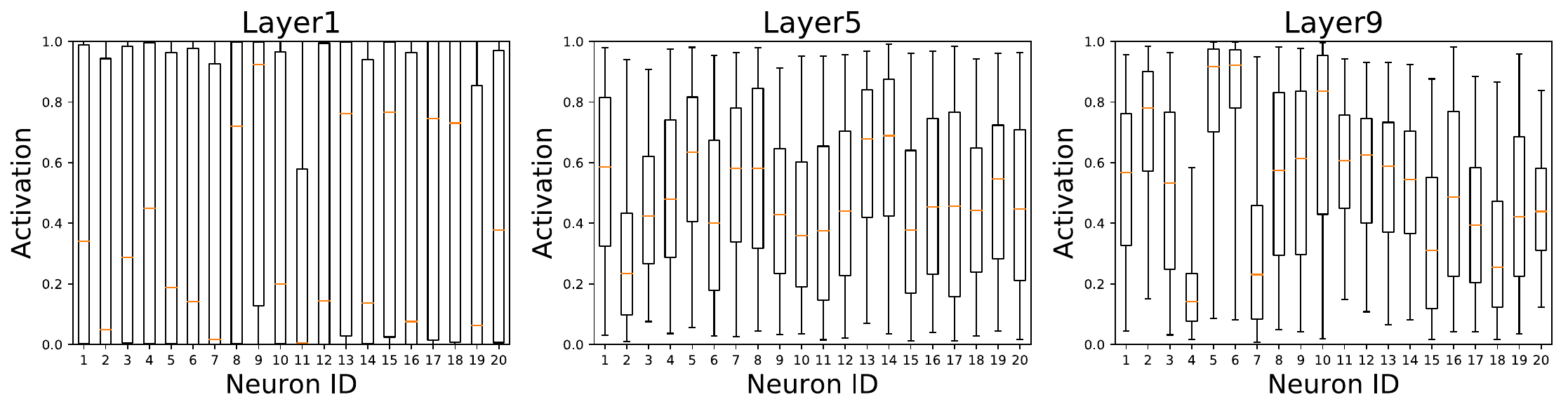}
        \label{fig:boxplot_activation_lcw_159_after10epochs}}
    \caption{Boxplot summaries of~$a_i^l$ on the first 20 neurons in layers
        1,5, and 9 of the 10-layer sigmoid MLP with LCW.}
    \label{fig:boxplot_activation_lcw_159}
\end{figure}
\begin{figure}[t]
    \centering
    \includegraphics[width=\linewidth, clip]{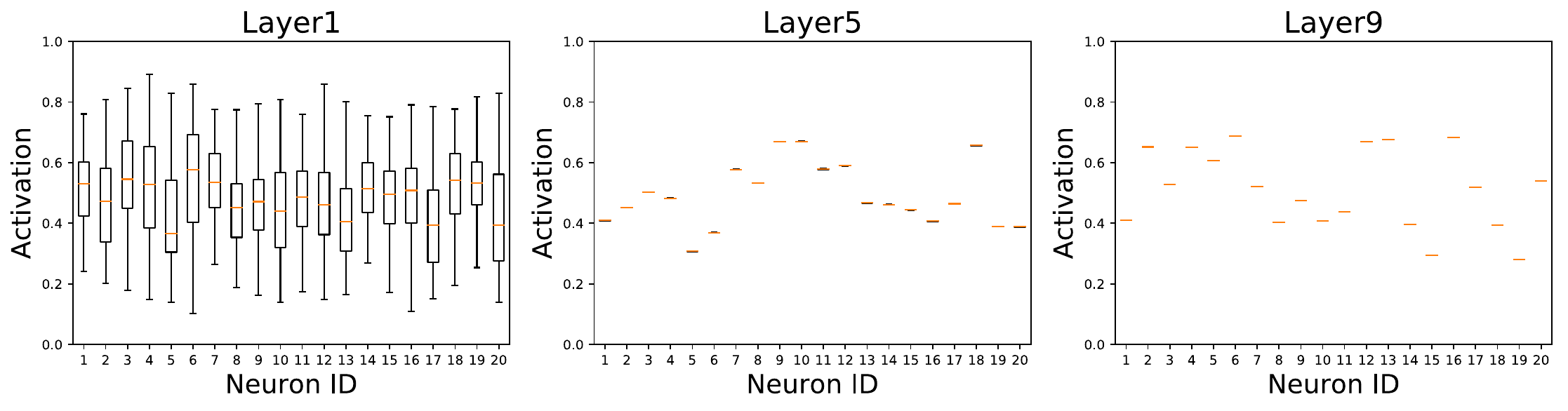}
    \caption{Boxplot summaries of~$a_i^l$ on neurons in layers
        1,5, and 9 of the 10-layer sigmoid MLP \textit{without} LCW, in which
        weights are initialized by the method in~\cite{glorot2010understanding}.}
    \label{fig:boxplot_activation_159}
\end{figure}

It is possible to force~$\bm{a}^l$ to follow~$\mathcal{P}_\gamma$ by applying
BN to preactivation~$z_i^l$. The distribution of~$z_i^l$ is then normalized to
have zero-mean and unit variance, and consequently, $a_i^l = f(z_i^l)$ are
more likely to follow the same distribution, indicating
that~$\bm{a}^l \sim \mathcal{P}_\gamma$ holds.
As will be discussed in Section~\ref{sec:related-work}, BN itself also
has an effect of reducing activation shift.
However, our experimental results suggest that we can train deep
networks more smoothly by combining LCW and BN, which will be shown in
Section~\ref{sec:experiments}.

\subsection{Learning LCW via Reparameterization}
\label{sec:learning-lcw-via}
A straightforward way to train a neural network with LCW is to solve a constrained
optimization problem, in which a loss function is minimized under the
condition that each weight vector is included
in~$\mathcal{W}_{\text{LC}}$.
Although several methods are available to solve such constrained problems, for example,
the gradient projection method~\cite{Luenberger:2015:LNP:2843008},
it might be less efficient to solve a constrained optimization problem than to
solve an unconstrained one.
We propose a reparameterization technique that enables us to train a
neural network with LCW using a solver for unconstrained
optimization.
The constraints on the weight vectors are embedded into
the structure of the neural network by the following reparameterization.

\noindent
\underline{Reparameterization}:
Let~$\bm{w}_i^l \in \mathbb{R}^m$ be a weight vector in a neural
network.
To apply LCW to~$\bm{w}_i^l$, we
reparameterize~$\bm{w}_i^l$ using
vector~$\bm{v}_i^l \in \mathbb{R}^{m-1}$ as~$\bm{w}_i^l = \bm{B}_m \bm{v}_i^l$,
where~$\bm{B}_m \in \mathbb{R}^{m \times (m-1)}$ is an orthonormal basis
of~$\mathcal{W}_{\text{LC}}$, written as a matrix of column vectors.

It is obvious that~$\bm{w}_i^l = \bm{B}_m \bm{v}_i^l \in \mathcal{W}_{\text{LC}}$.
We then solve the optimization problem in
which~$\bm{v}_i^l$ is considered as a new variable in place of~$\bm{w}_i^l$.
This optimization problem is unconstrained because~$\bm{v}_i^l \in \mathbb{R}^{m-1}$.
We can search for~$\bm{w}_i^l \in \mathcal{W}_{\text{LC}}$ by
exploring~$\bm{v}_i^l \in \mathbb{R}^{m-1}$.
The calculation of an orthonormal basis of~$\mathcal{W}_{\text{LC}}$ is
described in Appendix~B in the supplementary material.
Note that the proposed reparameterization can be
implemented easily and efficiently using modern frameworks for
deep learning based on GPUs.

\subsection{LCW for Convolutional Layers}
We consider a convolutional layer with $C_{out}$~convolutional
kernels. The size of each kernel is~$C_{in} \times K_h \times K_w$,
where~$C_{in}$, $K_h$, and~$K_w$ are
the number of the input channels, height of the kernel, and
width of the kernel, respectively.
The layer outputs $C_{out}$ channels of feature maps.
In a convolutional layer, activation shift occurs at the channel level,
that is, the preactivation has different mean value in each output
channel depending on the kernel of the channel.
We propose a simple extension of LCW for reducing the activation shift
in convolutional layers by introducing a
subspace~$\mathcal{W}_{\text{LC}}^{\text{kernel}}$
in~$\mathbb{R}^{C_{in} \times K_h \times K_w}$ defined as follows:
\begin{align}
    \mathcal{W}_{\text{LC}}^{\text{kernel}}:=\left\{\bm{w} \in
    \mathbb{R}^{C_{in} \times K_h \times K_w} \ \Biggl| \
    \sum_{i=1}^{C_{in}}\sum_{j=1}^{K_h}\sum_{k=1}^{K_w} w_{i,j,k} = 0\right\}, \notag
\end{align}
where~$w_{i,j,k}$ indicates the~$(i,j,k)$-th element of~$w$.
Subspace $\mathcal{W}_{\text{LC}}^{\text{kernel}}$ is a straightforward extension
of~$\mathcal{W}_{\text{LC}}$ to the kernel space.
To apply LCW to a convolutional layer, we restrict each kernel of the layer
in~$\mathcal{W}_{\text{LC}}^{\text{kernel}}$.

It is possible to apply the reparameterization trick described in the
previous subsection to LCW for convolutional layers.
We can reparameterize the kernel using an orthonormal basis
of~$\mathcal{W}_{\text{LC}}^{\text{kernel}}$
in which the kernel in~$\mathbb{R}^{C_{in} \times K_h \times K_w}$ is
unrolled into a vector of length~$C_{in} K_h  K_w$.

\section{Variance Analysis}
\label{sec:netw-arch-weight}
In this section, we first investigate the effect of removing activation shift
in a neural network based on an analysis of how the
variance of variables in the network changes through layer operations
both in forward and backward directions.
Then, we discuss its relationship to the vanishing gradient problem.

\subsection{Variance Analysis of a Fully Connected Layer}
\label{sec:vari-analys-fully}
The forward calculation of a fully connected layer
is~$\bm{z}^l = \bm{W}^l \bm{a}^{l-1} + \bm{b}^l$,
where~$\bm{W}^l=(\bm{w}_1^l, \ldots, \bm{w}_m^l)^\top$.
We express the $j$-th column vector of~$\bm{W}^l$ as~$\tilde{\bm{w}}_j^l$.
If we denote the gradient of a loss function with respect to
parameter~$v$ as~$\nabla_v$, the backward calculation
regarding~$\bm{a}^{l-1}$
is~$\nabla_{\bm{a}^{l-1}} = (\bm{W}^l)^\top \nabla_{\bm{z^l}}$.
The following proposition holds for the forward computation, in which
$\bm{I}_{m}$ is the identity matrix of order~$m \times m$,
$V$ indicates the variance, and~$Cov$ denotes the variance-covariance matrix.
\begin{proposition} \label{prop4}
    Assuming that~$\bm{w}_i^l \in \mathcal{W}_{\text{LC}}$,
    $E(\bm{a}^{l-1}) = \gamma_{\bm{a}^{l-1}} \bm{1}_m$
    with~$\gamma_{\bm{a}^{l-1}} \in \mathbb{R}$,
    $Cov(\bm{a}^{l-1})=\sigma_{\bm{a}^{l-1}}^2 \bm{I}_m$
    with~$\sigma_{\bm{a}^{l-1}} \in \mathbb{R}$,
    and~$\bm{b}^l = \bm{0}$,
    it holds that~$E(z_i^l) = 0$
    and~$V(z_i^l) = \sigma_{\bm{a}^{l-1}}^2 \|\bm{w}_i^l\|^2$.\footnote{A
        similar result is discussed in~\cite{huang2017centered}, but our result
        is more general because we do not
        assume the distribution of~$\bm{a}^{l-1}$ to be Gaussian distribution,
        which is assumed in~\cite{huang2017centered}.}
\end{proposition}
We also have the following proposition for the backward computation.
\begin{proposition} \label{prop5}
    Assuming that~$E(\nabla_{\bm{z}^{l}}) = \bm{0}$
    and~$Cov(\nabla_{\bm{z}^{l}})=\sigma_{\nabla_{\bm{z}^{l}}}^2\bm{I}_m$
    with~$\sigma_{\nabla_{\bm{z}^{l}}} \in \mathbb{R}$,
    it holds that~$E(\nabla_{a_j^{l-1}}) = 0$
    and~$V(\nabla_{a_j^{l-1}}) = \sigma_{\nabla_{\bm{z}^{l}}}^2 \|\tilde{\bm{w}}_j^l\|^2$.
\end{proposition}
For simplicity, we assume that~$\forall i, \|\bm{w}_i^l\|^2 = \eta^l$
and~$\forall j, \|\tilde{\bm{w}}_j^l\|^2 = \xi^l$.
Proposition~\ref{prop4} then indicates that, in the forward computation,
$V(z_i^l)$, the variance of the output,
becomes $\eta^l$~times larger than that of the input, $V(a_i^{l-1})$.
Proposition~\ref{prop5} indicates that,
in the backward chain, $V(\nabla_{a_i^{l-1}})$, the variance of the output,
becomes $\xi^l$~times larger than that of the input, $V(\nabla_{z_i^l})$.
If~$\bm{W}^l$ is a square matrix, then~$\eta^l = \xi^l$ (see Appendix~A
for proof), meaning
that the variance is amplified at the same rate in both the forward and
backward directions.
Another important observation is that, if we replace~$\bm{W}^l$
with~$\kappa \bm{W}^l$, the rate of amplification of the variance becomes
$\kappa^2$~times larger in both the forward and backward chains.
This property does not hold
if~$\bm{w}_i^l \not\in \mathcal{W}_{\text{LC}}$, because in this
case~$E(z_i^l) \neq 0$ because of the effect of the activation shift.
The variance is then more amplified in the forward chain than in the backward
chain by the weight rescaling.

\subsection{Variance Analysis of a Nonlinear Activation Layer}
\label{sec:vari-analys-nonls}
The forward and backward chains of the nonlinear activation layer are given
by~$a_i^l = f(z_i^l)$ and~$\nabla_{z_i^l} = f'(z_i^l) \nabla_{a_i^l}$, respectively.
The following proposition holds if~$f$ is the ReLU~\cite{5459469,icml2010_NairH10} function.
\begin{proposition} \label{prop6}
    Assuming that~$z_i^l$ and~$\nabla_{a_i^l}$ independently
    follow~$\mathcal{N}(0, \sigma_{z_i^l}^2)$
    and $\mathcal{N}(0, \sigma_{\nabla_{a_i^l}}^2)$, respectively,
    where~$\mathcal{N}(\mu, \sigma^2)$ indicates a normal distribution with
    mean~$\mu$ and variance~$\sigma^2$, it holds
    that
    \begin{align}
        V(a_i^l) = \frac{\sigma_{z_i^l}^2}{2}\left(1-\frac{1}{\pi}\right)
        \quad \text{and} \quad
        V(\nabla_{z_i^l}) = \frac{\sigma_{\nabla_{a_i^l}}^2}{2}. \notag
    \end{align}
\end{proposition}
We denote the rate of amplification of variance in the
forward and backward directions of a nonlinear activation function
by~$\phi_{\text{fw}} := V(a_i^l)/V(z_i^l)$
and~$\phi_{\text{bw}} := V(\nabla_{z_i^l})/V(\nabla_{a_i^l})$, respectively.
Proposition~\ref{prop6} then indicates that the variance is
amplified by a factor of~$\phi_{\text{fw}} = 0.34$ in
the forward chain and by a factor of~$\phi_{\text{bw}} = 0.5$ in the
backward chain through the ReLU activation layer.

If~$f$ is the sigmoid activation, there is no analytical solution for the
variance of~$a_i^l$ and~$\nabla_{z_i^l}$.
We therefore numerically examined~$\phi_{\text{fw}}$
and~$\phi_{\text{bw}}$ for the
sigmoid activation
under the conditions that~$z_i^l$ follows~$\mathcal{N}(0,\hat{\sigma}^2)$
for~$\hat{\sigma} \in \{0.5, 1, 2\}$ and~$\nabla_{a_i^l}$
follows~$\mathcal{N}(0, 1)$.
As a result, we
obtained~$(\phi_{\text{fw}}, \phi_{\text{bw}})=(0.236, 0.237)$,
$(0.208, 0.211)$, and~$(0.157, 0.170)$ for $\hat{\sigma} = 0.5, 1$,
and~$2$, respectively.
It suggests that the difference
between~$\phi_{\text{fw}}$ and~$\phi_{\text{bw}}$ in the sigmoid
activation layer decreases as the
variance of~$z_i^l$ decreases.

\subsection{Relationship to the Vanishing Gradient Problem}
\label{sec:relat-vanish-grad}
We consider an MLP in which the number of neurons is the
same in all hidden layers.
We initialize weights in the network by the method
based on minibatch statistics:
weights are first generated randomly, then rescaled so that the
preactivation in each layer
has unit variance on the minibatch of samples.
In fully connected layers with standard weights,
the variance of variables in the network is more amplified in the
forward chain than in the backward chain by the weight rescaling,
as discussed in Subsection~\ref{sec:vari-analys-fully}.
In contrast, in the sigmoid activation layers, the rate of amplification of the
variance is almost the same in the forward and backward directions, as
mentioned in the previous subsection.
Then, the variance of the preactivation gradient decreases
exponentially by rescaling the weights to maintain the variance of the
preactivation in the forward chain, resulting in the vanishing gradient,
that is, the preactivation gradient
in earlier layers has almost zero variance, especially
when the network have many layers.

In contrast,
when the LCW is applied to the network, the variance is amplified at the
same rate in both the forward and backward chains through fully
connected layers regardless of the weight rescaling.
In this case, the preactivation gradient has a similar variance in each
layer after the initialization, assuming that the sigmoid activation is used.
Concretely, the variance is amplified
by approximately~$0.21$ through the sigmoid activation layers in both
the forward and backward chains. Then, fully
connected layers are initialized to have the amplification rate
of~$1/0.21$ to keep the preactivation variance in the forward chain.
The gradient variance is then also amplified by~$1/0.21$ in
the backward chain of fully connected layers with LCW, indicating that the
gradient variance is also preserved in the backward chain.

From the analysis in the previous subsections, we also see that
normal fully connected layer and the ReLU layer have
opposite effect on amplifying the variance in each layer,
This may be another explanation why ReLU works well in practice without
techniques such as BN.

\subsection{Example}
For example, we use a $20$-layered MLP with sigmoid activation functions.
The weights of the MLP are initialized according to
the method described in the previous subsection.
We randomly took $100$~samples from the CIFAR-10 dataset and input them into the MLP.
The upper part of
Fig.~\ref{fig:mlp_20_256_sigmoid_gradients}~\subref{fig:mlp_20_256_sigmoid_gradient}
shows boxplot summaries of the preactivation in each layer.
The lower part shows boxplot summaries of the gradient with respect to the
preactivation in each layer, in which the standard cross-entropy loss is
used to obtain the gradient.
From Fig.~\ref{fig:mlp_20_256_sigmoid_gradients}~\subref{fig:mlp_20_256_sigmoid_gradient},
we see that the variance of the preactivation is preserved in the forward
chain, whereas the variance of the preactivation gradient rapidly
shrinks to zero in the backward chain, suggesting the vanishing gradient.
\begin{figure}[htbp]
    \centering
    \subfloat[MLP with standard weights.]{
        \includegraphics[width=0.8\linewidth, clip]{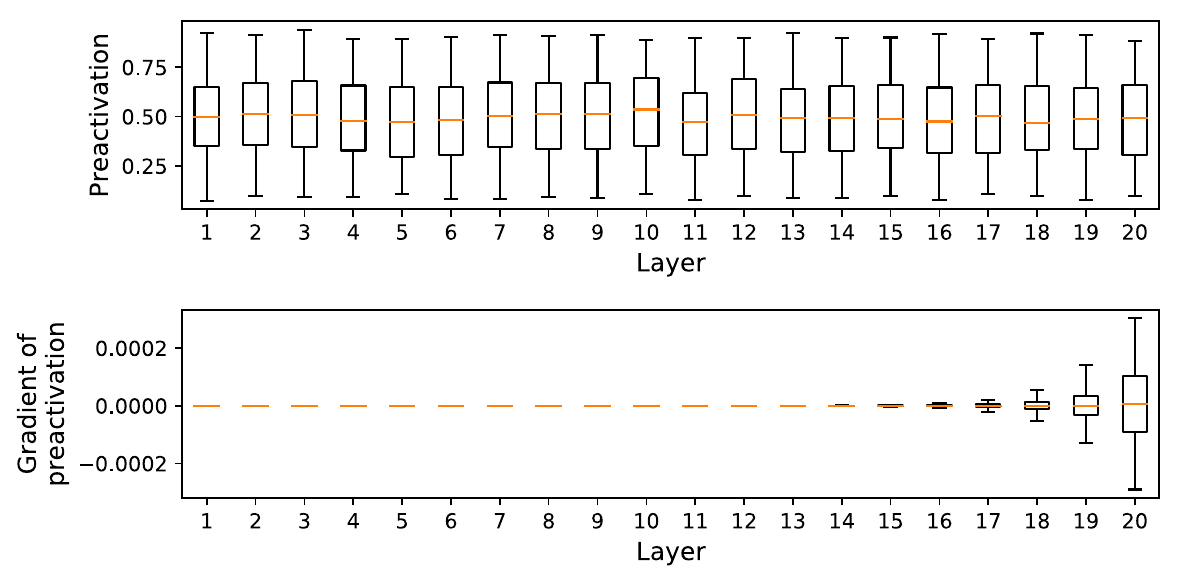}
        \label{fig:mlp_20_256_sigmoid_gradient}}
    \hfill%
    \subfloat[MLP with LCWs.]{
        \includegraphics[width=0.8\linewidth, clip]{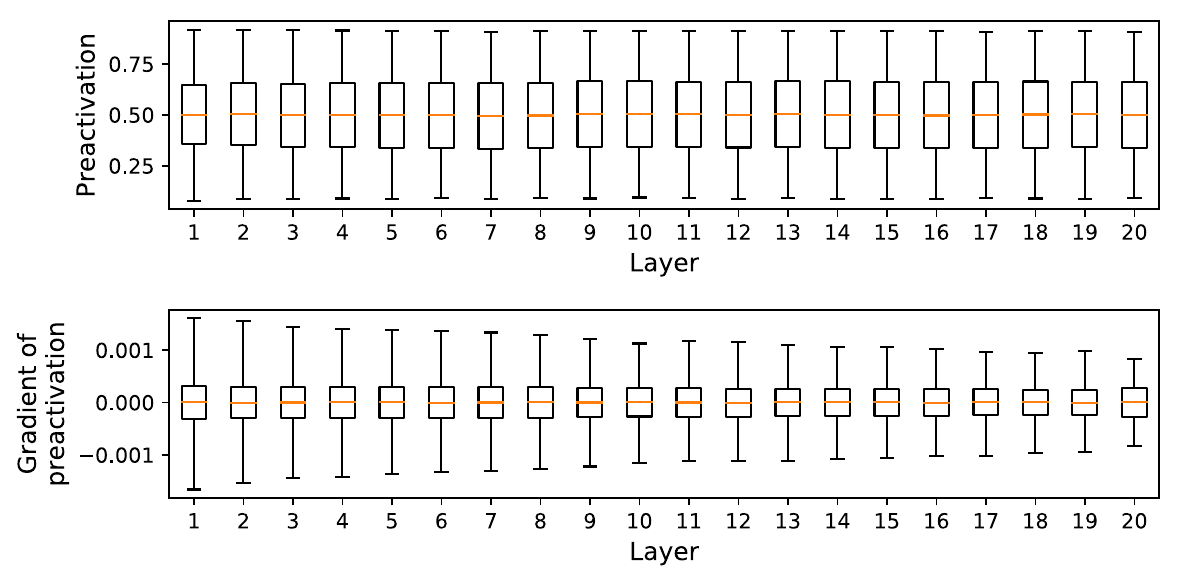}
        \label{fig:mlp_20_256_sigmoid_gradient_lcw}}
    \caption{Boxplot summaries of the preactivation (top) and its gradient (bottom)
        in 20-layered sigmoid MLPs with standard
        weights~(a) and LCWs~(b).}
    \label{fig:mlp_20_256_sigmoid_gradients}
\end{figure}

Next, LCW is applied to the MLP, and then, the weighs are initialized by
the same procedure.
Fig.~\ref{fig:mlp_20_256_sigmoid_gradients}~\subref{fig:mlp_20_256_sigmoid_gradient_lcw}
shows the distribution of the preactivation and its gradient in each
layer regarding the same samples from CIFAR-10.
In contrast to
Fig.~\ref{fig:mlp_20_256_sigmoid_gradients}~\subref{fig:mlp_20_256_sigmoid_gradient},
the variance of the preactivation gradient does not shrink to zero in
the backward chain.
Instead we observe that the variance of the gradient slightly increases
through the backward chain. This can be explained by the fact that the
variance is slightly more amplified in the backward chain than in the
forward chain through the sigmoid layer, as discussed in
Subsection~\ref{sec:vari-analys-nonls}.
These results suggest that we can resolve the vanishing gradient problem in an
MLP with sigmoid activation functions by applying LCW and by
initializing weights to preserve the preactivation variance in
the forward chain.

\section{Related work}
\label{sec:related-work}
Ioffe and Szegedy~\cite{icml2015_ioffe15} proposed the BN approach for accelerating
the training of deep nets.
BN was developed to address the problem of \textit{internal covariate shift}, that
is, training deep nets is difficult because the distribution of the input to
a layer changes as the weights of the preceding layers change during training.
BN is widely adopted in practice and shown to accelerate the training of
deep nets, although it has recently been argued that the success of BN
does not stem from the reduction of the internal covariate shift~\cite{NIPS2018_7515}.
BN computes the mean and standard deviation of~$z_i^{l}$
based on a minibatch, and then, normalizes~$z_i^{l}$ by using these
statistics.
G\"{u}l\c{c}ehre and Bengio~\cite{JMLR:v17:gulchere16a} proposed the
\textit{standardization layer~(SL)} approach, which is similar to BN.
The main difference is that SL normalizes~$a_i^{l}$, whereas BN
normalizes~$z_i^{l}$.
Interestingly, both BN and SL can be considered mechanisms for reducing the
activation shift. On one hand, SL reduces the activation shift by
forcing~$\|\hat{\bm{\mu}}\| = 0$ in
Proposition~\ref{prop2}. On the other hand, BN reduces the activation
shift by removing the mean from~$z_i^{l}$ for each neuron.
A drawback of both BN and SL is that the model has to be switched during
inference to ensure that its output depends only on the
input and not the minibatch.
In contrast, the LCW proposed in this paper do not require any change
in the model during inference.

Salimans and Kingma~\cite{salimans2016weight} proposed \textit{weight normalization~(WN)}
in which a weight vector~$\bm{w}_i^{l} \in \mathbb{R}^m$ is reparameterized
as~$\bm{w}_i^{l} = (g_i^l/\|\bm{v}_i^l\|) \bm{v}_i^l$, where~$g_i^l \in \mathbb{R}$
and~$\bm{v}_i^l \in \mathbb{R}^m$ are new parameters.
By definition, WN does not have the property of reducing the activation shift,
because the degrees of freedom of~$\bm{w}_i^l$ are unchanged by the
reparameterization.
They also proposed a minibatch-based initialization by
which weight vectors are initialized so that~$z_i^{l}$ has
zero mean and unit variance,
indicating that the activation shift is resolved immediately after the initialization.
Our preliminary results presented in Section~\ref{sec:experiments} suggest that
to start learning with initial weights that do not incur activation shift
is not sufficient to train very deep nets.
It is important to incorporate a mechanism that reduces the activation shift
during training.

Ba et al.~\cite{layernorm} proposed \textit{layer normalization~(LN)} as a
variant of BN.
LN normalizes~$z_i^{l}$ over the neurons in a layer on a sample in
a minibatch,
whereas BN normalizes~$z_i^{l}$ over the minibatch on a neuron.
From the viewpoint of reducing the activation shift, LN is not as direct
as BN.
Although LN does not resolve the activation shift, it should normalize
the degree of activation shift in each layer.

Huang et al.~\cite{huang2017centered} proposed \textit{centered weight
    normalization (CWN)} as an extension of WN, in which
parameter~$\bm{v}_i^l$ in WN is reparameterized
by~$\bm{v}_i^l = \tilde{\bm{v}}_i^l - \bm{1}_m (\bm{1}_m^\top \tilde{\bm{v}}_i^l)/m$
with~$\tilde{\bm{v}}_i^l \in R^m$.
CWN therefore forces a weight vector~$\bm{w}_i^l$ to satisfy
both~$\|\bm{w}_i^l\|=1$ and~$\bm{1}_m^\top \bm{w}_i^l = 0$.
CWN was derived from the observation that, in practice, weights in a
neural network are initially sampled from a distribution with
zero-mean.
CWN and LCW share the idea of restricting weight vectors so that they
have zero mean during training, although they come from different
perspectives and have different implementations.
The main differences between CWN and LCW are the following:
CWN forces weight vectors to have both unit norm and zero mean, whereas
LCW only forces the latter from the analysis that the latter
constraint is essential to resolve the activation shift;
LCW embeds the constraint into the network structure using the
orthonormal basis of a subspace of weight vectors;
the effect of reducing activation shift by introducing LCW is analyzed
from the perspective of variance amplification in both
the forward and backward chains.

Miyato et al.~\cite{spectral_norm} proposed
\textit{spectral normalization (SN)} that constrains
the spectral norm, that is, the largest singular value, of a weight
matrix equal to 1.
SN was introduced to control the Lipschitz constant
of the discriminator in the GAN framework~\cite{NIPS2014_5423} to
stabilize the training.
The relationship between the spectral norm of weights and the
generalization ability of deep nets is discussed
in~\cite{spectral_regularization}.
However, controlling the spectral norm of weight matrices is orthogonal to
the reduction of the activation shift.

He et al.~\cite{he2016deep} proposed \textit{residual network} that
consists of a stack of residual blocks with skip connections.
If we denote the input to the $l$-th residual block
by~$\bm{x}^l \in \mathbb{R}^m$,
the output~$\bm{x}^{l+1}$, which is the input to the next residual
block, is given by~$\bm{x}^{l+1} = \bm{x}^l + \bm{\mathcal{F}}_l(\bm{x}^l)$,
where~$\bm{\mathcal{F}}_l:\mathbb{R}^m \to \mathbb{R}^m$ is a
mapping defined by a stack of nonlinear layers.
In contrast to the original residual network that regard the activation
as~$\bm{x}^l$, He et al.~\cite{he2016identity} proposed
\textit{preactivation structure} in which the
preactivation is regarded as~$\bm{x}^l$.
Residual network will indirectly reduce the impact of the activation
shift. The reason is explained below: In a residual network,
it holds that
$\bm{x}^L = \bm{x}^0 + \sum_{l=0}^{L-1} \bm{\mathcal{F}}_l(\bm{x}^l)$.
The activation shift can occur in each of~$\bm{\mathcal{F}}_l(\bm{x}^l)$,
that is, each output element of~$\bm{\mathcal{F}}_l(\bm{x}^l)$ has
different mean.
However, the shift pattern is almost random in
each~$\bm{\mathcal{F}}_l(\bm{x}^l)$, and consequently,
the mean shift in~$\bm{x}^L$ can be moderate because it is the average over
these random shifts.
This may be another reason why residual networks are successful in
training deep models.

\section{Experiments}
\label{sec:experiments}

We conducted experiments using the CIFAR-10
and CIFAR-100 datasets~\cite{krizhevsky2009learning},
which consist of color natural images each of which is annotated
corresponding to 10 and 100 classes of objects, respectively.
We preprocessed each dataset by subtracting the channel means and
dividing by the channel standard deviations.
We adopted standard data augmentation~\cite{he2016deep}: random cropping
and horizontal flipping.

All experiments were performed using Python~3.6 with PyTorch~0.4.1~\cite{paszke2017automatic}
on a system running Ubuntu 16.04 LTS with GPUs.
We implemented LCW using standard modules equipped with PyTorch.
As implementation of BN, SL, WN, and SN, we employed corresponding libraries
in PyTorch. We implemented CWN by modifying modules for WN.

\subsection{Deep MLP with Sigmoid Activation Functions}
\label{sec:exp-feed-forward-network}
We first conducted experiments using an MLP model with 50 hidden layers, each
containing 256 hidden units with sigmoid activation functions,
followed by a softmax layer combined with a cross-entropy
loss function.
We applied each of LCW, BN, SL, WN, CWN, and SN to the model, and compared
the performance.
We also considered models with each of the above techniques (other than
BN) combined with BN.
These models are annotated with, for example, ``BN+LCW'' in the results.

Models with LCW were initialized following the method described in
Section~\ref{sec:relat-vanish-grad}.
Models with WN or CWN were initialized according to~\cite{salimans2016weight}.
Models with BN, SL, or SN were initialized using the method proposed
in~\cite{glorot2010understanding}.
Each model was trained using a stochastic gradient descent with a
minibatch size of~$128$, momentum of~$0.9$, and weight decay of~$0.0001$
for $100$~epochs.
The learning rate starts from 0.1 and is multiplied by 0.95 after every
epoch until it reaches the lower threshold of 0.001.

Fig.~\ref{fig:result_mlp_50_256} shows the curve of training loss, test
loss, training accuracy, and test accuracy of each model on each
dataset, in which the horizontal axis shows the training epoch.
The results of MLPs with WN or SN are omitted in
Fig.~\ref{fig:result_mlp_50_256}, because the training of these models
did not proceed at all.
This result matches our expectation that reducing the activation shift
is essential to train deep neural networks, because WN and SN themselves
do not have the effect of reducing activation shift as discussed in
Section~\ref{sec:related-work}.
\begin{figure}[htbp]
    \centering
    \subfloat[Results for the CIFAR-10 dataset.]{
        \includegraphics[width=0.8\linewidth, clip]{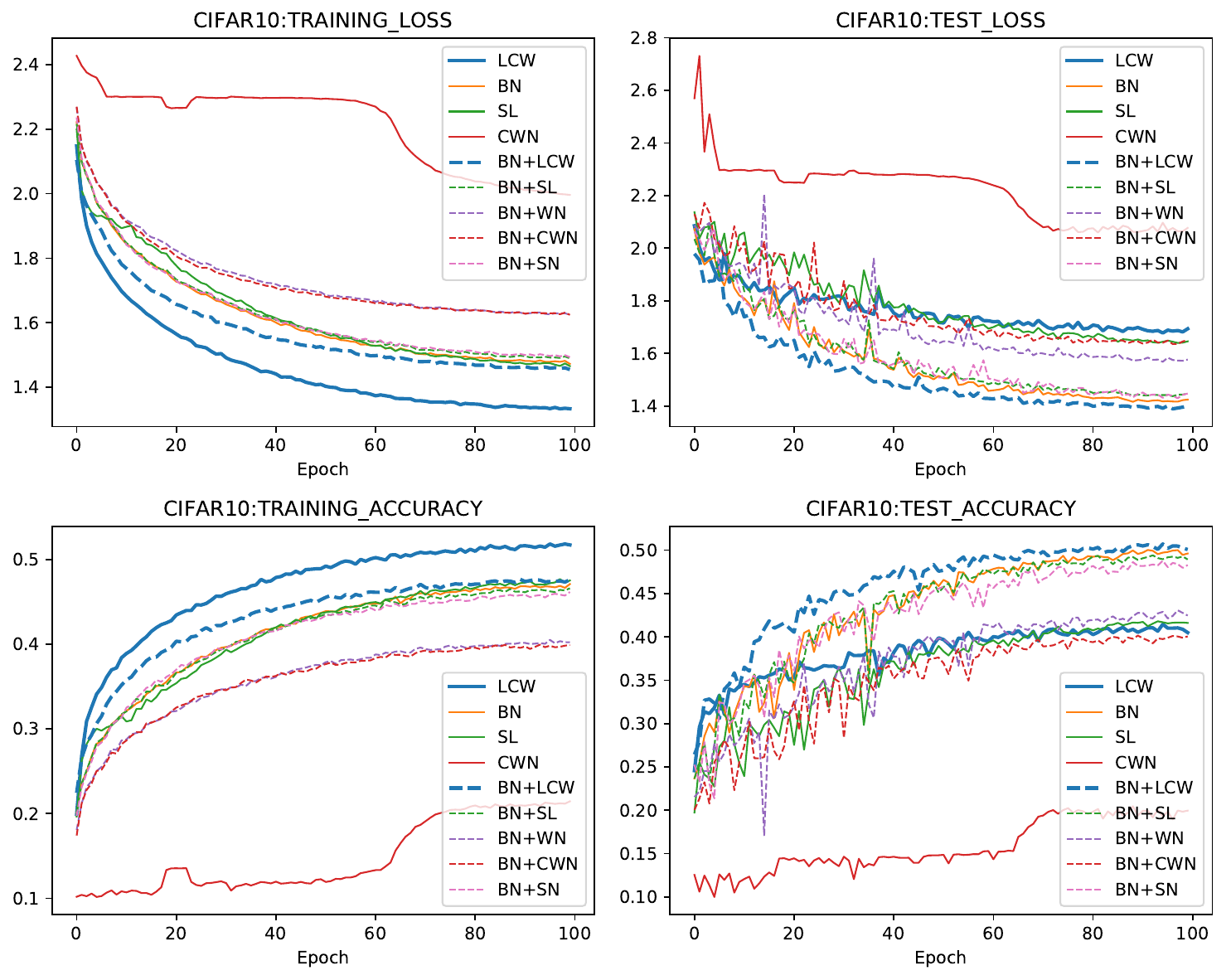}
        \label{fig:result_mlp_50_256_cifar10}}
    \hfill%
    \subfloat[Results for the CIFAR-100 dataset.]{
        \includegraphics[width=0.8\linewidth, clip]{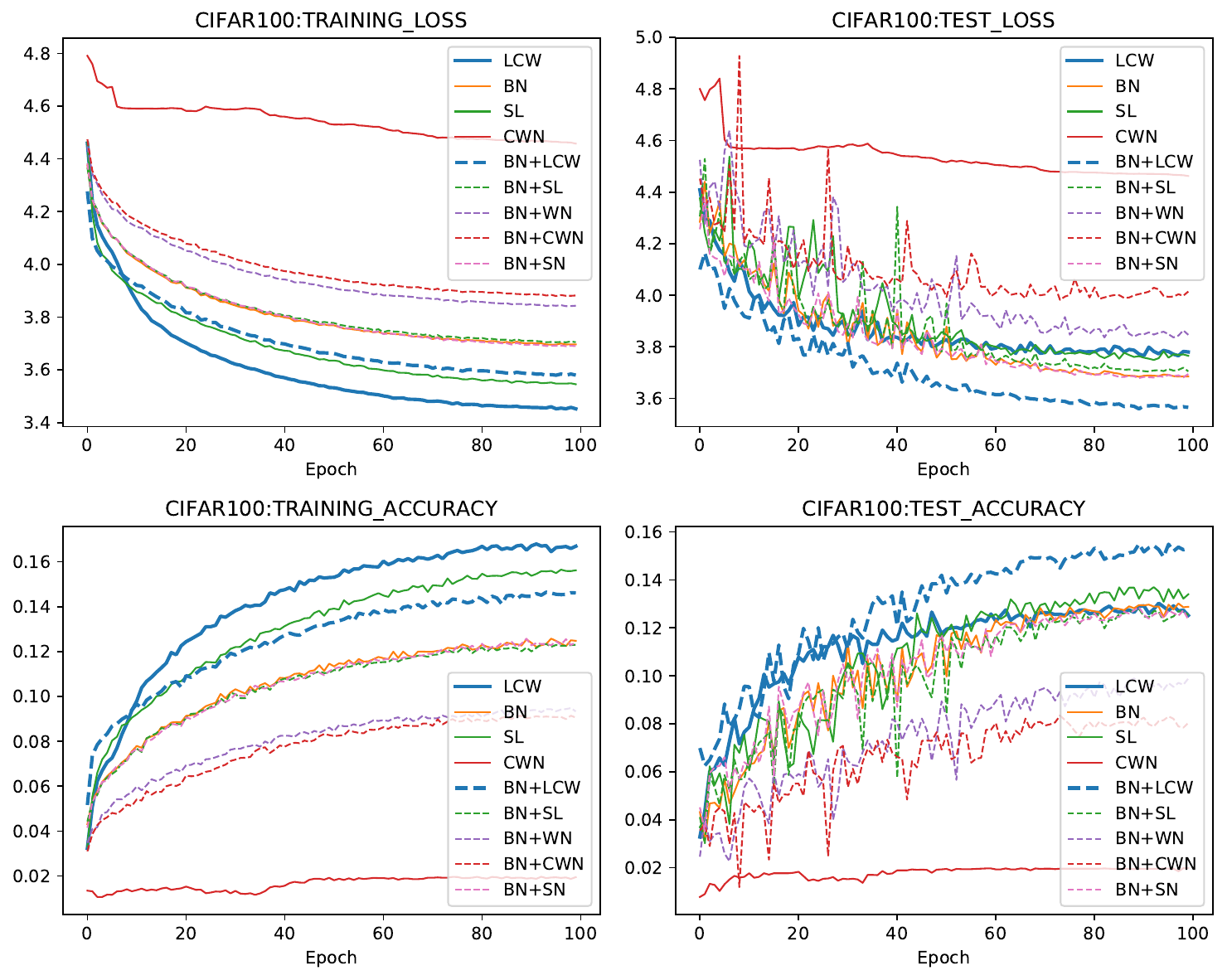}
        \label{fig:result_mlp_50_256_cifar100}}
    \caption{Training loss (upper left), test loss (upper right), training
        accuracy (lower left), and test accuracy (lower right) of 50-layer MLPs
        for CIFAR-10 (a) and CIFAR-100 (b).}
    \label{fig:result_mlp_50_256}
\end{figure}
We see that LCW achieves higher rate of convergence and gives better scores
with respect to the training loss/accuracy compared with other models.
However, with respect to the test loss/accuracy, the scores of LCW
are no better than that of other models.
This result suggests that LCW has an ability to accelerate the network
training but may increase the risk of overfitting.
In contrast, combined with BN, LCW achieves better performance in test
loss/accuracy, as shown by the results annotated with ``BN+LCW'' in
Fig.~\ref{fig:result_mlp_50_256}.
We think such improvement was provided because LCW accelerated the
training while the generalization ability of BN was maintained.

\subsection{Deep Convolutional Networks with ReLU Activation Functions}
\label{sec:exp-resid-conv-netw}
In this subsection, we evaluate LCW using convolutional
networks with ReLU activation functions.
As base models, we employed the following two models:

\noindent
\underline{VGG19}:
A 19-layer convolutional network in which 16 convolutional layers are
connected in series, followed by three fully connected layers with
dropout~\cite{vgg}.
We inserted BN layers before each ReLU layer in VGG19,
although the original VGG model does not include BN layers.\footnote{This
    is mainly because VGG was proposed earlier than BN.}

\noindent
\underline{ResNet18}:
An 18-layer convolutional network with residual
structure~\cite{he2016deep}, which consists of eight
residual units each of which contains two convolutional layers in the
residual part.
We employed the full preactivation structure proposed
in~\cite{he2016identity}.
In ResNet18, BN layers are inserted before each ReLU layer.

We applied LCW, WN, CWN, or SN to VGG19 and ResNet18, respectively, and
compared the performance including the plain VGG19 and ResNet18 models.
Each model was trained using a stochastic gradient descent with a
minibatch size of 128, momentum of 0.9, and weight decay of 0.0005.
For the CIFAR-10 dataset, we trained each model for 300~epochs with the
learning rate that starts from 0.1 and is multiplied by 0.95 after every
three epochs until it reaches 0.001.
For the CIFAR-100 dataset, we trained each model for 500~epochs with the
learning rate multiplied by 0.95 after every five epochs.

Table~\ref{tab:conv_results} shows the test accuracy and loss for the
CIFAR-10 and CIFAR-100 datasets, in which each value was evaluated as the
average over the last ten epochs of training.
\begin{table}[tb]
    \centering
    \caption{Test accuracy/loss of convolutional models for CIFAR-10 and
        CIFAR-100 datasets.}
    \label{tab:conv_results}
    {\tabcolsep = 3.7pt
        \begin{tabular}{lrrp{0pt}rr}
            \hline
                                              & \multicolumn{2}{c}{CIFAR-10}      &                               & \multicolumn{2}{c}{CIFAR-100}                 \\
            \cline{2-3} \cline{5-6}
            Model                             & \multicolumn{1}{c}{Test Accuracy} & \multicolumn{1}{c}{Test Loss}

                                              &                                   &
            \multicolumn{1}{c}{Test Accuracy} & \multicolumn{1}{c}{Test Loss}

            \\
            \hline
            VGG19                             & 0.936                             & 0.354                         &                               & 0.732 & 1.788 \\
            VGG19+LCW                         & 0.938                             & 0.332                         &                               & 0.741 & 1.569 \\
            VGG19+WN                          & 0.931                             & 0.391                         &                               & 0.725 & 1.914 \\
            VGG19+CWN                         & 0.934                             & 0.372                         &                               & 0.727 & 1.827 \\
            VGG19+SN                          & 0.936                             & 0.358                         &                               & 0.733 & 1.644 \\
            \hline
            ResNet18                          & 0.952                             & 0.204                         &                               & 0.769 & 0.978 \\
            ResNet18+LCW                      & 0.952                             & 0.187                         &                               & 0.770 & 0.955 \\
            ResNet18+WN                       & 0.951                             & 0.206                         &                               & 0.777 & 0.947 \\
            ResNet18+CWN                      & 0.948                             & 0.216                         &                               & 0.781 & 0.949 \\
            ResNet18+SN                       & 0.952                             & 0.206                         &                               & 0.780 & 1.015 \\
            \hline
        \end{tabular}
    }
\end{table}
We see that LCW improves the generalization performance of VGG19 with
respect to both the test accuracy and loss. The improvement is more
evident for the CIFAR-100 dataset.
The curve of training loss and accuracy of VGG19-based models for
CIFAR-100 are shown in Fig.~\ref{fig:result_VGG19_cifar100}.
We see that LCW enhances the rate of convergence, which we think
lead to the better performance.
\begin{figure}[t]
    \centering
    \includegraphics[width=0.8\linewidth, clip]{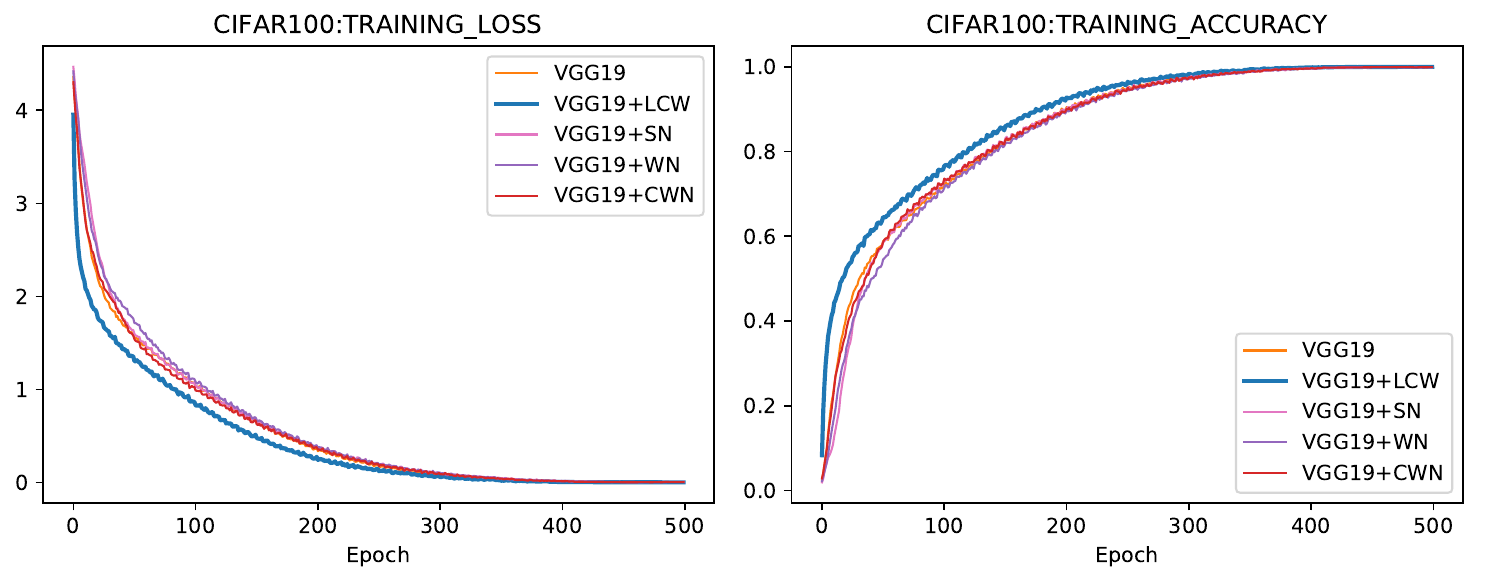}
    \caption{Training loss (left) and training accuracy (right) of the
        VGG19-based models for the CIFAR-100 dataset.}
    \label{fig:result_VGG19_cifar100}
\end{figure}
In contrast, the improvement brought by LCW is less evident in
ResNet18, in particular, with respect to the test accuracy.
We observed little difference in the training curve of ResNet18 with and
without LCW.
A possible reason for this is that the residual structure itself has an
ability to mitigate the impact of the activation shift, as discussed in
Section~\ref{sec:related-work}, and therefore the reduction of
activation shift by introducing LCW was less beneficial for ResNet18.

\section{Conclusion}
\label{sec:conclusion}
In this paper, we identified the activation shift in a neural network:
the preactivation of a neuron has non-zero mean depending on the angle
between the weight vector of the neuron and the mean of the activation
vector in the previous layer.
The LCW approach was then proposed to reduce the activation shift.
We analyzed how the variance of variables in a neural network changes
through layer operations in both forward and backward chains, and
discussed its relationship to the vanishing gradient problem.
Experimental results suggest that the proposed method works well in a
feedforward network with sigmoid activation functions, resolving the
vanishing gradient problem.
We also showed that existing methods that successfully accelerate the training of deep
neural networks, including BN and residual structures,
have an ability to reduce the effect of activation shift, suggesting
that alleviating the activation shift is essential for efficient
training of deep models.
The proposed method achieved better
performance when used in a convolutional network with ReLU activation
functions combined with BN.
Future work includes investigating the applicability of the proposed
method for other neural network structures, such as recurrent structures.
%
%
\bibliographystyle{plainnat}
\bibliography{references}

\newpage
\appendix
\pagestyle{empty}
\section{Proofs}
\begin{proof}[Proof of Proposition~\ref{prop1}]
    It follows from
    \[
        E(\bm{w}_i^l \cdot \bm{a}^{l-1}) = \bm{w}_i^l \cdot E(\bm{a}^{l-1}) = \bm{w}_i^l \cdot (\gamma \bm{1}_m) =
        \|\bm{w}_i^l\|\ \|\gamma \bm{1}_m\| \cos \theta_i^l = |\gamma|\sqrt{m} \|\bm{w}_i^l\| \cos \theta_i^l,
    \]
    where $E(x)$ denotes the expected value of random variable~$x$. 
\end{proof}
\begin{proof}[Proof of Proposition~\ref{prop2}]
    The proof is the same as that of Proposition~1. 
\end{proof}
\begin{proof}[Proof of Proposition~\ref{prop3}]
    $E(\bm{w}_i^l \cdot \bm{a}^{l-1}) = \bm{w}_i^l \cdot E(\bm{a}^{l-1}) = \gamma \left(
        \bm{w}_i^l \cdot \bm{1}_m \right) =  0$. 
\end{proof}
\begin{proof}[Proof of Proposition~\ref{prop4}]
    \begin{align*}
        E(z_i^l) & = \bm{w}_i^l \cdot E(\bm{a}^{l-1}) =
        \gamma_{\bm{a}^{l-1}} \bm{w}_i^l \cdot \bm{1}_m = 0,                   \\
        V(z_i^l) & = \left(\bm{w}_i^l\right)^\top Cov(\bm{a}^{l-1}) \bm{w}_i^l
        = \left(\bm{w}_i^l\right)^\top
        \left(\sigma_{\bm{a}^{l-1}}^2\bm{I}_m\right) \bm{w}_i^l
        = \sigma_{\bm{a}^{l-1}}^2 \|\bm{w}_i^l\|^2.
    \end{align*} 
\end{proof}
\begin{proof}[Proof of Proposition~\ref{prop5}]
    \begin{align*}
        E(\nabla_{a_j^{l-1}}) & = \tilde{\bm{w}}_j^l \cdot E(\nabla_{\bm{z}^l}) =
        0,                                                                        \\
        V(\nabla_{a_j^{l-1}}) & = \left(\tilde{\bm{w}}_j^l\right)^\top
        Cov(\nabla_{\bm{z}^l}) \tilde{\bm{w}}_j^l
        = \left(\tilde{\bm{w}}_j^l\right)^\top
        \left(\sigma_{\nabla_{\bm{z}^{l}}}^2 \bm{I}_m\right) \tilde{\bm{w}}_j^l
        = \sigma_{\nabla_{\bm{z}^{l}}}^2 \|\tilde{\bm{w}}_j^l\|^2.
    \end{align*} 
\end{proof}
\begin{proof}[Proof of ``$\eta^l = \xi^l$ when $\bm{W}^l$ is a square matrix'' in Section~\ref{sec:vari-analys-nonls}]
    From the assumption of $\forall i, \|\bm{w}_i^l\|^2 = \eta^l$
    and~$\forall j, \|\tilde{\bm{w}}_j^l\|^2 = \xi^l$,
    we have~$\sum_{i=1}^m \|\bm{w}_i^l\|^2 = m\eta^l$
    and~$\sum_{j=1}^m \|\tilde{\bm{w}}_j^l\|^2 = m\xi^l$, respectively.
    From the fact that~$\sum_{i=1}^m \|\bm{w}_i^l\|^2 = \sum_{j=1}^m
        \|\tilde{\bm{w}}_j^l\|^2 = \sum_{i=1}^m \sum_{j=1}^m (w_{i,j}^l)^2$,
    where~$w_{i,j}^l$ indicates the $(i,j)$-th element of~$\bm{W}^l$, we
    have~$m \eta^l = m \xi^l$. Consequently, it holds that~$\eta^l = \xi^l$.
\end{proof}
\begin{proof}[Proof of Proposition~\ref{prop6}]
    With the assumption that~$f$ is the ReLU activation function and~$z_i^l$
    follows~$\mathcal{N}(0, \sigma_{z_i^l}^2)$, $a_i^l = f(z_i^l)$ follows
    the mixture distribution
    \begin{align}
        0.5 \times \mathcal{N}(0, 0)
        + 0.5 \times \mathcal{TN}_{(0, \infty)}(0, \sigma_{z_i^l}^2), \notag
    \end{align}
    where~$\mathcal{TN}_{(0, \infty)}(0, \sigma_{z_i^l}^2)$ denotes the
    truncated normal distribution: $\mathcal{N}(0, \sigma_{z_i^l}^2)$ bounded
    within the interval~$(0, \infty)$.
    The mean and variance of~$\mathcal{TN}_{(0, \infty)}(0, \sigma_{z_i^l}^2)$
    are given by~$\sigma_{z_i^l} \sqrt{\frac{2}{\pi}}$
    and~$\sigma_{z_i^l}^2 \left(1 - \frac{2}{\pi}\right)$, respectively.
    The variance of~$a_i^l$ is then given as follows:\footnote{The variance of a
        random variable that follows the mixture of distributions A and B with
        mixture weights~$p_A$ and~$p_B$, respectively, is given
        by
        \begin{align}
            p_A \sigma_A^2 + p_B \sigma_B^2 + p_A p_B (\mu_A - \mu_B)^2, \notag
        \end{align}
        where~$\mu_A$ and~$\mu_B$ are the mean of distributions A and B,
        respectively, and~$\sigma_A^2$ and~$\sigma_B^2$ are the variance of
        distributions A and B, respectively.}
    \begin{align}
        V(a_i^l) & =  0.5 \times 0 + 0.5 \times \sigma_{z_i^l}^2 \left(1 - \frac{2}{\pi}\right) + 0.5^2 \times \left(\sigma_{z_i^l} \sqrt{\frac{2}{\pi}}\right)^2 \notag \\
                 & = \frac{\sigma_{z_i^l}^2}{2} \left(1 - \frac{1}{\pi}\right). \notag
    \end{align}

    In contrast, from the assumption that~$f$ is the ReLU function
    and~$z_i^l$ follows~$\mathcal{N}(0, \sigma_{z_i^l}^2)$, $f'(z_i^l)$ follows
    a Bernoulli distribution with probability~$p=0.5$ whose mean and
    variance are 0.5 and 0.25, respectively.
    The variance of~$\nabla_{z_i^l} = f'(z_i^l)\nabla_{a_i^l}$ is then given
    as follows:\footnote{Given
        two independent random variables~$X$ and~$Y$, the variance of~$XY$ is
        given by
        \begin{align}
            V(XY) = V(X) V(Y) + E^2(X) V(Y) + E^2(Y) V(X). \notag
        \end{align}}
    \begin{align}
        V(\nabla_{z_i^l}) & = 0.25 \times \sigma_{\nabla_{a_i^l}}^2 + 0.5^2
        \times \sigma_{\nabla_{a_i^l}}^2 + 0 \times 0.25 \notag             \\
                          & = \frac{\sigma_{\nabla_{a_i^l}}^2}{2} \notag.
    \end{align} 
\end{proof}
\section{Calculating an orthonormal basis of $\mathcal{W}_{LC}$}
For calculating an orthonormal basis of $\mathcal{W}_{LC}$,
we consider a
matrix~$\tilde{\bm{B}}_m \in \mathbb{R}^{m \times \left(m-1\right)}$
whose column vectors span~$\mathcal{W}_{LC}$.
For example, $\tilde{\bm{B}}_m$ is given by
\begin{align}
    \tilde{\bm{B}}_m = \left(\begin{array}{c}
                                     \bm{I}_{m-1} \\
                                     -\bm{1}_{m-1}^\top
                                 \end{array}\right). \notag
\end{align}
Given~$\tilde{\bm{B}}_m$,
we apply QR decomposition to factorize~$\tilde{\bm{B}}_m$
into~$\bm{Q}$ and~$\bm{R}$, where~$\bm{Q} \in \mathbb{R}^{m \times \left(m-1\right)}$
is a matrix whose columns are the orthonormal basis for~$\mathcal{W}_{LC}$
and~$\bm{R} \in \mathbb{R}^{\left(m-1\right) \times \left(m-1\right)}$
is an upper triangular matrix.
We can utilize~$\bm{Q}$ as~$\bm{B}_m$ in Section~\ref{sec:learning-lcw-via}.
%
\end{document}